\documentclass[11pt]{article}

\usepackage[preprint]{acl}

\usepackage{times}
\usepackage{latexsym}

\usepackage[T1]{fontenc}

\usepackage[utf8]{inputenc}

\usepackage{microtype}

\usepackage{inconsolata}

\usepackage{graphicx}

\usepackage{amsmath} 
\usepackage{amsfonts} 

\usepackage{amsthm}
\theoremstyle{plain}            
\newtheorem{lemma}{Lemma}  

\usepackage{algorithm}
\usepackage{algorithmic}

\title{Survival-Guided Length Control for Efficient Diffusion Language Models}

\author{Ivan Kobyzev\thanks{Equal contribution}\hspace{3mm}Abbas Ghaddar\footnotemark[1]\hspace{3mm}Yufei Cui \\
Huawei Noah’s Ark Lab, Montreal Research Center, Canada\\
{\normalsize\texttt{\{ivan.kobyzev,abbas.ghaddar,yufei.cui\}@huawei.com}}}

\usepackage{amssymb}
\usepackage{pifont}

\newcommand{\mightmention}[1]{}
\newcommand{\problem}[1]{\textcolor{red}{$\star$}}
\newcommand{\answer}[1]{\textcolor{blue}{$\#$}}
\newcommand{\todoreview}[1]{\textcolor{green}{$@$}}

\usepackage{subcaption}
\usepackage{rotating,csquotes}
\usepackage{xcolor}
\usepackage{multirow}
\usepackage{framed}
\usepackage{booktabs}
\usepackage{amsmath}
\usepackage{upgreek}
\usepackage{amsfonts}

\usepackage[most]{tcolorbox}

\newtcbox{\mybox}[1][]{enhanced, colframe=blue, colback=blue!15, 
	frame style={opacity=0.25}, interior style={opacity=0.25}, 
	nobeforeafter, tcbox raise base, shrink tight, extrude by=1mm, #1}

\usepackage{pgfplots}
\usepackage{array}
\usepackage{xcolor}
\usetikzlibrary{shapes.multipart}

\usepackage{subcaption}
\usepackage{rotating,csquotes}
\usepackage{xcolor}
\usepackage{multirow}
\usepackage{framed}
\usepackage{booktabs}
\usepackage{amsmath}
\usepackage{upgreek}
\usepackage{amsfonts}
\usepackage{arydshln}

\usepackage{amsmath,amssymb,mathtools}

\usepackage{enumitem}

\usepackage{float}

\usepackage{rotating}     
\usepackage{subcaption}   
\usepackage{booktabs}
\usepackage{caption}

\makeatletter
\AtBeginDocument{%
}
\makeatother

\begin{document}
\maketitle
\begin{abstract}

Diffusion language models (DLMs) generate text by iteratively denoising masked sequences, but standard decoding either fixes the sequence length or relies on ad hoc stopping rules, often leading to unnecessary denoising steps. We recast length selection as a discrete-time survival problem over the end-of-sequence token and propose a plug-in, training-free length predictor that can be added to any existing DLM. Across reasoning and code-generation benchmarks, survival-guided length decoding speeds up inference by up to $7\times$ while preserving task accuracy. We further find that predicted lengths vary widely even within the same dataset, making model performance sensitive to the chosen length.
\end{abstract}

\section{Introduction}
Masked Diffusion Language Models (DLMs) \cite{austin2021structured,shi2024simplified} generate text by iteratively denoising a masked canvas.
At each diffusion step, masked positions are updated using the model’s predictive distribution, gradually transforming an all-\texttt{[MASK]} suffix into the text.
This iterative refinement enables flexible any-order decoding and parallelism, but it also raises a core question for decoding cost: when to stop denoising and which tokens to commit at each step?

Standard decoding for masked DLMs often follows an any-order autoregressive (AOAR) pattern \cite{Ou2024YourAD}: choose a large task-agnostic maximum length $L_{\max}$, run a fixed denoising schedule, and stop only when all masks are removed or an iteration budget is reached.
While simple, this approach is frequently wasteful: many prompts require far fewer tokens than $L_{\max}$, yet the decoder still spends computation refining positions that correspond to an unnecessarily long canvas.

As a result, a major source of overhead in masked DLM inference is not the denoising rule itself, but the mismatch between a conservative global length budget and the instance-specific length needed.
In this work, we focus on length selection: for a given prompt, how many new tokens should we generate?
We show that length selection admits a natural interpretation in terms of survival analysis~\cite{survival_andersen1993counting}.

Concretely, we treat generation length as a discrete-time survival variable over the end-of-sequence token \texttt{[EOS]}.
Using the DLM’s per-position \texttt{[EOS]} probabilities from a single forward pass on a long masked canvas, we obtain a plug-in estimate of the discrete-time hazard and recover a closed-form estimate of the expected length via standard survival identities.
This yields a \textit{training-free, model-agnostic length predictor} that can be plugged into existing masked DLMs without modifying model parameters or changing the underlying denoising schedule. Empirically, survival-guided length prediction reduces unnecessary refinement 
of positions beyond the true end of the sequence
and speeds up inference without sacrificing performance relative to baselines that decode with a sufficiently large $L_{\max}$.

Extensive experiments with two large-scale DLMs, LLaDA~\cite{llada} and Dream~\cite{dream}, consistently demonstrate inference speedups of up to $7\times$ across reasoning and code-generation benchmarks, with no loss in task performance. Additionally, our analysis suggests that length prediction is indeed sample-dependent, with lengths varying significantly across samples of the same distribution.

\section{Method}

Survival analysis~\citep{survival_andersen1993counting} studies the distribution of random event times, classically in settings such as time-to-failure or time-to-death. A central object is the hazard function, which models the probability that an event occurs at time $t$ given that it has not occurred before. This viewpoint is naturally discrete and sequential: at each step we either survive to the next step or terminate. Applying this framework to DLMs, the unknown sequence length becomes a discrete event time for the end-of-sequence token, and at each candidate position we ask if the sequence has not ended yet, how likely is it to end here? 

This allows us to reuse standard survival identities to obtain a plug-in estimate of the expected length from a single pass of the DLM. We consider masked diffusion language models~\cite{austin2021structured,shi2024simplified} that generate text by iteratively denoising a masked canvas.
Let
\begin{equation}
\label{eq:canvas}
    \mathbf{x}^{(0)} = [x_{1:P}, \texttt{[MASK]},\dots,\texttt{[MASK]}]
\end{equation}
be the initial sequence consisting of a prompt of length $P$ followed by $T$ masked positions, and let $\mathbf{x}^{(s)}$ denote the sequence after $s$ denoising steps.
At each step $s$, the DLM produces logits and probabilities:
\[
\mathbf{z}^{(s)} = f_\theta\bigl(\mathbf{x}^{(s)}, t_s\bigr),
\quad
p^{(s)}_{i,v} = \text{softmax}_v z^{(s)}_{i,v},
\]
for position $i$ and vocabulary token $v$, where $t_s$ is the (discrete) diffusion time.

We treat the unknown sequence length as a discrete-time survival problem over positions in the generation span. Consider the masked input to the DLM as in Eq.~\ref{eq:canvas}, 
where the prompt occupies the positions $1,\dots,P$, and the model is tasked to generate up to $T$ new tokens in positions $P+1, \dots, P+T$. 
However, the ideal generation length is unknown. In standard practice $T$ is set to a large, task-agnostic upper bound $L_{\max}$, so that many prompts terminate long before $L_{\max}$ but the decoder still runs the full denoising schedule on all $T$ slots.
This mismatch between a fixed global budget and the ideal instance-specific length is a major source of wasted compute.

To predict the ideal generation length $L$, 
let us run one pass of the diffusion model on the initial canvas and denote the resulting logits as $\mathbf{z}^{(0)}$.
 The quality of the text decoded after one step of DLM is usually extremely low, but these logits are still informative and can be used for our task of length prediction if we interpret its per-position \texttt{[EOS]} probabilities as a length-survival signal. 

Let $t$ be the position of the logits $\mathbf{z}^{(0)}$ in the sequence, and define

\begin{equation}
\begin{aligned}
  p_t 
  &:= \mathbb{P} \big(\texttt{[EOS]} \text{ at position } t \,\big|\, \text{prompt, masks} \big) \\
  &= \mathrm{softmax}(\mathbf{z}^{(0)}_t)[\texttt{[EOS]}],
\end{aligned}
\end{equation}
for  $t = P{+}1,\dots,P{+}T.$
We model the sequence termination with a discrete-time survival process. Define the hazard at relative position $k = 1,\dots,T$ as the probability
that the sequence ends now, given that it has not ended before:
$$ h_k := \mathbb{P}(L = k \mid L \ge k, \text{prompt}). $$ 
Then assuming that the diffusion model is well-trained and fits the data distribution, we can estimate the hazard as $h_k \approx  p_{P+k}$. 

We make a mean-field-style approximation~\cite{Blei2016VariationalIA}, and treat the \texttt{[EOS]} events across positions as conditionally independent given the prompt.
Under this assumption, the survival up to step $k$ is
\begin{equation}
   S(k) \;=\; \prod_{i=1}^{k} \bigl(1 - h_i\bigr),
  \qquad S(0) = 1, 
\end{equation}
and the probability of termination exactly at $k$ is
\begin{equation}
    \pi_k \;=\; \mathbb{P}(L = k)
  \;\approx\; h_k\,S(k-1).
\end{equation}
  
Given the plug-in length distribution $\{\pi_k\}_{k=1}^T$, we can compute the truncated expected length:
\begin{align}
  \mathbb{E}[L] 
  \;=\; \sum_{k=1}^{T} k\,\pi_k
  \;=\; \sum_{k=1}^{T} S(k-1),
  \label{eq:exp-from-survival}
\end{align}
where the second equality is the standard identity that the expectation equals the sum of survivals (see proof in Appendix~\ref{app:exp-survival}).

At test time we choose a single plug-in length estimate $\hat{L}(x) = \mathbb{E}[L]$ and pass it to the DLM as the maximum number of new tokens.
Crucially, this requires no extra training or parameters. The length predictor is entirely derived from the base model’s \texttt{[EOS]} logits. See Algorithm~\ref{alg:length-prediction} in Appendix~\ref{app:algo_len}.

\section{Experiments}

\subsection{Experimental Setting}
We experiment with two powerful masked diffusion language models (DLMs) that are trained in notably different ways: LLaDA-8B-Base~\cite{llada} and Dream-v0-Base-7B~\cite{dream}. Both are implemented in PyTorch~\cite{paszke2019pytorch} on top of the Transformers library~\cite{wolf2020transformers}. Hereafter, we refer to these two models as LLaDA and Dream, respectively.

We consider a set of reasoning and code-generation benchmarks used in the original evaluations of both models, including BBH~\cite{suzgun2023challenging}, GSM8K~\cite{cobbe2021training}, MATH~\cite{hendrycks2measuring}, HumanEval~\cite{chen2021evaluating}, and MBPP~\cite{austin2021program}. 
We follow the standard LM Evaluation Harness~\cite{eval-harness} setup, reporting 3-shot accuracy on BBH, strict match 5-shot accuracy on GSM8K, 4-shot accuracy on MATH, pass@1 on HumanEval (0-shot)\footnote{For Dream on HumanEval we apply the post-processing following the authors' public implementation at https://github.com/DreamLM/Dream.}, and pass@1 on MBPP (3-shot).
We use each model's custom evaluation code and per-benchmark hyperparameters.\footnote{Authors of both LLaDA and Dream use a suffix-masked sequence length of $L_{\max}\!=\!1024$ for all benchmarks.} All experiments are conducted on a single modern compute accelerator, and we report results with a batch size of 1.

\subsection{Main Results}

\begin{table}[!ht]
    \begin{center}
    \resizebox{\columnwidth}{!}{
        \begin{tabular}{lcc|cc}
            \toprule
            \multirow{2}{*}{\textbf{Task}} & \multicolumn{2}{c}{\textbf{LLaDA}} & \multicolumn{2}{c}{\textbf{Dream}} \\
            \cmidrule(lr){2-3} \cmidrule(lr){4-5}
             & \textbf{w/o} & \textbf{w/} & \textbf{w/o} & \textbf{w/} \\
            \midrule
            BBH       & 73   & 14   ($5.2\times$) & 56.3 & 10.5 ($5.4\times$) \\
            GSM8K     & 91   & 16   ($5.2\times$) & 81.4 & 12.3 ($6.6\times$) \\
            MATH      & 76   & 23   ($3.3\times$) & 64.5 & 16.1 ($4.0\times$) \\
            HumEval & 54   & 9    ($6.0\times$) & 53.0 & 16.6 ($3.2\times$) \\
            MBPP      & 80   & 26   ($3.1\times$) & 64.8 & 16.6 ($3.9\times$) \\
            \bottomrule
        \end{tabular}
        }
    \end{center}
    \caption{Decoding speed (seconds per sample) for baseline (\textbf{w/o}) vs.\ length prediction (\textbf{w/}). Parentheses report speedup relative to the baseline.}
    \label{tab:main_res}
\end{table}
\vspace{-2mm}

Table~\ref{tab:main_res} shows the decoding speed (in terms of seconds per example) for LLaDA and Dream, when using standard AOAR decoding with a fixed maximum length (\textbf{w/o}) against one that is equipped with our survival-guided length predictor (\textbf{w/}). We observe that, for both models, using the predicted length yields substantial decoding-speed improvements across all benchmarks, with speedups ranging from $3.2\times$ to $6.6\times$. In addition, we find that on most benchmarks (except HumanEval), the speedup is in a similar range for both models despite their structural differences.

Table~\ref{tab:main_perf} shows the per-task performances, as well as standard deviation, of models with and without our length predictor. As one can see, the efficiency gains reported in  \autoref{tab:main_res} do not come at the expense of task performance. Across all benchmarks, the differences between \textbf{w/o} and \textbf{w/} remain within the reported standard deviations, indicating no statistically meaningful change.

\begin{table}[!ht]
    \centering
    \setlength{\tabcolsep}{3.5pt} 
    \renewcommand{\arraystretch}{0.95} 
    \resizebox{\columnwidth}{!}{
    \begin{tabular}{lcc|cc}
        \toprule
        \multirow{2}{*}{\textbf{Task}} & \multicolumn{2}{c}{\textbf{LLaDA}} & \multicolumn{2}{c}{\textbf{Dream}} \\
        \cmidrule(lr){2-3} \cmidrule(lr){4-5}
     & \textbf{w/o} & \textbf{w/} & \textbf{w/o} & \textbf{w/} \\
        \midrule
        BBH            & $49.5\!\pm\!0.6$ & $51.8\!\pm\!0.6$ & $63.7\!\pm\!0.5$ & $63.8\!\pm\!0.5$ \\
        GSM8K & $69.9\!\pm\!1.3$ & $70.0\!\pm\!1.3$ & $74.5\!\pm\!1.2$ & $74.6\!\pm\!1.2$ \\
        MATH           & $31.6\!\pm\!0.6$ & $31.0\!\pm\!0.6$ & $39.6\!\pm\!0.7$ & $39.5\!\pm\!0.7$ \\
        HumEval      & $32.3\!\pm\!3.7$ & $33.5\!\pm\!3.7$ & $58.0\!\pm\!3.7$ & $59.0\!\pm\!3.7$ \\
        MBPP           & $40.0\!\pm\!2.0$ & $40.2\!\pm\!2.1$ & $58.0\!\pm\!2.0$ & $58.0\!\pm\!2.0$ \\
        \bottomrule
    \end{tabular}}
    \caption{Task performance and standard deviation for baseline (\textbf{w/o}) vs.\ length prediction (\textbf{w/}).}
    \label{tab:main_perf}
\end{table}
\vspace{-3mm}

This indicates that our survival-guided length predictor reliably trims redundant tail steps while leaving the quality of the generated sequence essentially unchanged. A notable aspect of these results is that the procedure is completely model-agnostic and training-free. It consists of a single forward pass on a long masked canvas, interpretation of per-position \texttt{[EOS]} probabilities as a discrete-time hazard, and computing a closed-form expected length from the survival curve. Despite architectural and training differences, this plug-in estimator consistently delivers $3$–$7\times$ speedups with no measurable loss in accuracy, supporting the view that survival-guided length control is a robust primitive for efficient decoding in diffusion language models.

\subsection{Fixed Mean-Length Ablation}

\begin{table}[!ht]
    \centering
    \setlength{\tabcolsep}{3.5pt}
    \renewcommand{\arraystretch}{0.95}
    \resizebox{\columnwidth}{!}{
    \begin{tabular}{lcc|cc}
        \toprule
         \multirow{2}{*}{\textbf{Task}} & \multicolumn{2}{c}{\textbf{LLaDA}} & \multicolumn{2}{c}{\textbf{Dream}} \\
        \cmidrule(lr){2-3} \cmidrule(lr){4-5}
        & $\bar{L}$ & $\Delta$ & $\bar{L}$ & $\Delta$ \\
        \midrule
        BBH     & $537\!\pm\!204$ & $-3.2\!\pm\!0.7$ & $482\!\pm\!176$ & $-2.7\!\pm\!0.6$ \\
        GSM8K   & $278\!\pm\!29$  & $-2.8\!\pm\!1.8$  & $270\!\pm\!52$  & $-2.6\!\pm\!1.8$ \\
        MATH    & $498\!\pm\!64$  & $-0.6\!\pm\!0.8$  & $390\!\pm\!62$  & $-0.7\!\pm\!0.9$ \\
        HumEval & $382\!\pm\!57$  & $-1.8\!\pm\!5.2$  & $396\!\pm\!107$ & $-2.3\!\pm\!5.2$ \\
        MBPP    & $495\!\pm\!44$  & $0.0\!\pm\!3.0$   & $378\!\pm\!95$  & $-2.0\!\pm\!3.0$ \\
        \bottomrule
    \end{tabular}
    }
    \caption{We report the per-task mean predicted length ($\bar{L}$) along with its standard deviation, as well as the performance change ($\Delta$) and its standard deviation 
    when using $\bar{L}$ as a fixed horizon, compared to our per-sample survival-guided length selection.}
    
    \label{tab:fixed_length_ablation}
\end{table}
\vspace{-2mm}

To test whether our gains come simply from choosing a shorter global budget, rather than from instance-wise adaptation, we run an ablation where each model uses a \emph{single} fixed generation length equal to the dataset-wide mean predicted length $\bar{L}$ from our method.
Table~\ref{tab:fixed_length_ablation} reports the mean predicted length ($\bar{L}$) and its standard deviation for each model and benchmark. It also shows the performance change ($\Delta$) when replacing our per-sample length prediction with a fixed horizon set to $\bar{L}$, rather than using the maximum length $L_{\max}$. We observe that the fixed-length policy based on the mean predicted length generally underperforms survival-guided decoding across most settings, with a few exceptions (e.g., LLaDA on MBPP). We attribute this to substantial per-sample variation in required decoding length within each dataset, which exists even in carefully curated standard benchmarks. This observation suggests that per-sample length prediction is important for maintaining performance, as both models appear sensitive to the chosen decoding horizon.\footnote{More results and ablations are presented in Appendix~\ref{app:additional-results}.}

\section{Related Work}

\paragraph{Diffusion language models.}
Discrete denoising diffusion models for text were introduced in work on structured diffusion in discrete state spaces~\citep{austin2021structured} and later specialized into diffusion-style language models (DLMs) such as LLaDA and Dream, which iteratively unmask tokens on a fixed canvas~\citep{llada,dream}. 
These models trade the strictly left-to-right factorization of autoregressive LMs for parallel refinement and flexible token orders, but typical decoders still assume a conservative, fixed maximum length $T$ and a fixed number of refinement steps. 
Recent work has focused mainly on improving generation quality and scaling to larger models, with less attention to length prediction and training-free decoding policies for efficiency.

\paragraph{Length control for diffusion LMs.}

Closest to our length predictor is DAEDAL~\citep{li2025beyondfixed}, which also targets the fixed-length limitation of DLMs.
DAEDAL starts from a short canvas and repeatedly reruns the model with expanded lengths, using the average \texttt{[EOS]} confidence in a trailing window and a hand-tuned threshold to decide when the current length is sufficient.
Thus length is recovered by an iterative search over candidate horizons and requires several extra forward passes beyond the main decoding trajectory.
In contrast, our method runs the DLM once on a generously long canvas, interprets per-position \texttt{[EOS]} probabilities as a discrete-time hazard, and uses a standard survival identity to obtain a closed-form expected length, with no length loop or per-task thresholds.

A different line of work, exemplified by Prophet~\citep{li2025prophet}, studies ``early answer convergence'' in DLMs: the model often internally identifies the correct answer well before the last refinement step.
Prophet uses the gap between the top-2 candidates to decide when to commit all remaining tokens in one shot.
This addresses when to stop refinement for a fixed-length canvas, not how long the sequence itself should be, and it is based on confidence gaps over content tokens rather than \texttt{[EOS]} hazards.
Our length estimator is complementary: it predicts the output length before decoding, while leaving the subsequent refinement schedule (including Prophet-style early stops) unchanged.

\paragraph{Sampling strategies and unmasking schedules.}
There is a line of work to accelerate the decoding of DLMs with efficient unmasking schedulers. 
Standard sampling for masked diffusion models follows the MaskGIT paradigm~\citep{chang2022maskgit}, where at each step the model updates a top-$k$ or top-$p$ subset of masked positions by confidence. 
Adaptive schedulers such as EB-Sampler~\citep{benhamu2025accelerated} and SlowFast Sampling~\citep{wei2025slowfast} dynamically vary which tokens to refine based on entropy, convergence, or span structure, while other works focus on architectural and caching improvements~\citep{liu2025dllmcache,wu2025fastdllm}. All of these methods operate on a fixed-length canvas and are orthogonal to our method; in principle, they could be combined with our length predictor to further improve the decoding efficiency.

\section{Conclusion and Future Work}

We propose a training-free, survival-guided length predictor for diffusion language models that estimates the expected output length from a single forward pass. Across LLaDA and Dream, this plug-in length control delivers $3$--$7\times$ decoding speedups on reasoning and code benchmarks without statistically meaningful changes in accuracy. 
Beyond length control, the same survival perspective suggests a broader family of decoding rules for diffusion LMs, for example by viewing token-level commitment and refinement decisions through a survival or risk-allocation lens. Developing such token-level survival models and combining them with our length predictor in a unified, learnable framework is an interesting direction for future work on efficient and reliable diffusion decoding.

\section*{Limitations}
Like most efficient decoding methods for Diffusion Language Models, our approach is limited to a particular setting and several extensions remain open.
First, we evaluate survival-guided length control on two large-scale masked DLMs (LLaDA-8B-Base and Dream-7B-Base) and a standard set of reasoning and code-generation benchmarks (BBH, GSM8K, MATH, HumanEval, MBPP). This follows common evaluation practice in recent Diffusion LM works, but it does not cover settings such as extremely long-context generation, or multi-turn dialogue, so efficiency in those regimes remains to be verified. 
Second, our objective is strictly computational efficiency at inference time under a fixed model and decoding rule. We do not modify training, nor attempt to optimize for other properties of the outputs (e.g., calibration, robustness, or bias). Understanding how survival-guided decoding interacts with these broader aspects of model behavior is an interesting direction for future work.

\section*{Acknowledgments}
We thank the anonymous reviewers for their insightful comments.

\bibliography{custom}

\clearpage       
\appendix

\section{Expectation as a sum of survivals}
\label{app:exp-survival}

Let $L$ be an integer-valued random variable taking values in $\{1,\dots,T\}$ with
\[
\begin{aligned}
\pi_k &= \mathbb{P}(L = k), \qquad \\
S(k)  &= \mathbb{P}(L > k), \quad k = 0,\dots,T.
\end{aligned}
\]

Recall that in our construction $S(k) = \prod_{i=1}^{k} (1 - h_i)$ and hence
$S(k-1) = \mathbb{P}(L \ge k)$.

\begin{lemma}
    \begin{equation*}
     \mathbb{E}[L]
  \;=\; \sum_{k=1}^{T} k \,\pi_k
  \;=\; \sum_{k=1}^{T} S(k-1).
    \end{equation*}
\end{lemma}

\begin{proof}
By definition of mathematical expectation,
\[
  \mathbb{E}[L]
  = \sum_{k=1}^{T} k \,\pi_k.
\]
Then we can write an integer $k$ as a sum of $k$ ones:
\[
  k = \sum_{j=1}^{k} 1,
\]
and substitute into the expectation:
\[
  \mathbb{E}[L]
  = \sum_{k=1}^{T} \left( \sum_{j=1}^{k} 1 \right) \pi_k.
\]
After swapping the order of summation we get:
\[
  \mathbb{E}[L]
  = \sum_{j=1}^{T} \sum_{k=j}^{T} \pi_k.
\]
The inner sum is the tail probability
\[
  \sum_{k=j}^{T} \pi_k
  = \mathbb{P}(L \ge j) =\mathbb{P}(L > j-1) 
  = S(j-1).
\]
Therefore
\[
  \mathbb{E}[L]
  = \sum_{j=1}^{T} S(j-1),
\]
which is exactly the claim of the Lemma.
\end{proof}

\section{Algorithm for length prediction plug-in}
\label{app:algo_len}

\begin{algorithm*}[t]
\caption{Survival-guided length prediction for diffusion LMs}
\label{alg:length-prediction}
\begin{algorithmic}[1]
\REQUIRE Prompt tokens $\mathbf{x}_{1:P}$, maximal horizon $T$, EOS token id
\ENSURE Predicted number of new tokens $\hat{L} \in \{1,\dots,T\}$

\STATE \textbf{Input canvas.} Form the initial canvas
\[
  \mathbf{x}^{(0)} = [\mathbf{x}_{1:P}, \texttt{[MASK]},\dots,\texttt{[MASK]}]
\]
of total length $P + T$ (prompt + $T$ masked positions).

\STATE \textbf{Single forward pass.} Run the diffusion LM once on $x^{(0)}$ at the initial diffusion time to obtain logits
\[
  \mathbf{z}^{(0)}_t, \qquad t = 1,\dots,P+T.
\]

\STATE \textbf{EOS probabilities.} For each candidate generation position $t = P{+}1,\dots,P{+}T$, compute
\[
  p_t \gets \mathrm{softmax}(\mathbf{z}^{(0)}_t)[\texttt{[EOS]}].
\]

\STATE \textbf{Plug-in hazards.} For relative index $k = 1,\dots,T$, set
\[
  h_k \gets p_{P+k}
\]
as the plug-in estimate of the hazard $P(L = k \mid L \ge k, \text{prompt})$.

\STATE \textbf{Survival curve.} Initialize $S(0) \gets 1$. For $k = 1,\dots,T$:
\[
  S(k) \gets S(k-1)\,\bigl(1 - h_k\bigr).
\]

\STATE \textbf{Truncated expectation.}
Compute
\[
  \hat{L}_\text{cont}
  \;\gets\;
  \sum_{k=1}^{T} S(k-1) .
\]

\STATE \textbf{Post-processing.} Clip and round to an integer horizon:
\[
  \hat{L}_\text{cont} \gets \min\bigl(\max(\hat{L}_\text{cont}, 1), T\bigr), \qquad
  \hat{L} \gets \big\lceil \hat{L}_\text{cont} \big\rceil.
\]
\RETURN $\hat{L}$
\end{algorithmic}
\end{algorithm*}

\clearpage

\section{Additional Results}
\label{app:additional-results}

\subsection{Comparison with DAEDAL}
\label{app:daedal-instruct}

Conceptually, DAEDAL also addresses length, but it relies on an iterative length-expansion procedure with a tuned stopping threshold, and its stopping criterion depends on how reliably \texttt{[EOS]} is produced under the model’s decoding regime. In our experiments on DLMs-Base, \texttt{[EOS]} is often not emitted as an argmax early enough to support such thresholding reliably, whereas our estimator uses soft \texttt{[EOS]} probabilities and therefore applies uniformly to both base and instruction-tuned variants. This is a conceptual advantage of our approach over DAEDAL, which is only applicable to instruction-tuned models. Nevertheless, we evaluate against DAEDAL on instruction-tuned variants of LLaDA and Dream. Tables~\ref{tab:inst-speed-daedal} and~\ref{tab:inst-acc-daedal} report decoding time and task accuracy, respectively.

\begin{table}[!th]
\centering
\footnotesize
\setlength{\tabcolsep}{4pt}
\begin{tabular}{lccc}
\toprule
\textbf{Task} & \textbf{w/o} & \textbf{DAEDAL} & \textbf{Ours} \\
\midrule
\multicolumn{4}{c}{\textbf{LLaDA (Instruct)}} \\
\midrule
GSM8K     & 89 (1.0$\times$) & 28 (3.2$\times$) & 14 (6.4$\times$) \\
MATH      & 75 (1.0$\times$) & 41 (1.8$\times$) & 20 (3.8$\times$) \\
HumanEval & 53 (1.0$\times$) & 22 (2.4$\times$) & 8 (6.6$\times$) \\
MBPP      & 78 (1.0$\times$) & 68 (1.1$\times$) & 23 (3.4$\times$) \\
\midrule
\multicolumn{4}{c}{\textbf{Dream (Instruct)}} \\
\midrule
GSM8K     & 80 (1.0$\times$) & 22 (3.6$\times$) & 11 (7.3$\times$) \\
MATH      & 64 (1.0$\times$) & 38 (1.7$\times$) & 15 (4.3$\times$) \\
HumanEval & 53 (1.0$\times$) & 20 (2.7$\times$) & 14 (3.8$\times$) \\
MBPP      & 63 (1.0$\times$) & 54 (1.2$\times$) & 15 (4.2$\times$) \\
\bottomrule
\end{tabular}
\caption{Decoding speed (seconds per sample) on instruction-tuned models for DAEDAL and our length prediction method. Parentheses report speedup relative to the baseline decoder without (w/o) length prediction.}
\label{tab:inst-speed-daedal}
\end{table}

\begin{table}[!th]
\centering
\footnotesize
\setlength{\tabcolsep}{4pt}
\begin{tabular}{lccc}
\toprule
\textbf{Task} & \textbf{w/o} & \textbf{DAEDAL} & \textbf{Ours} \\
\midrule
\multicolumn{4}{c}{\textbf{LLaDA (Instruct)}} \\
\midrule
GSM8K     & 78.6 $\pm$ 1.4 & 79.2 $\pm$ 1.3 & 79.5 $\pm$ 1.5 \\
MATH      & 26.6 $\pm$ 0.7 & 26.9 $\pm$ 0.8 & 27.1 $\pm$ 0.7 \\
HumanEval & 47.6 $\pm$ 3.6 & 48.3 $\pm$ 3.5 & 48.2 $\pm$ 3.7 \\
MBPP      & 34.2 $\pm$ 2.1 & 34.8 $\pm$ 2.2 & 35.5 $\pm$ 2.0 \\
\midrule
\multicolumn{4}{c}{\textbf{Dream (Instruct)}} \\
\midrule
GSM8K     & 81.4 $\pm$ 1.3 & 81.5 $\pm$ 1.4 & 82.1 $\pm$ 1.2 \\
MATH      & 39.2 $\pm$ 0.8 & 39.7 $\pm$ 0.9 & 40.1 $\pm$ 0.8 \\
HumanEval & 55.5 $\pm$ 3.9 & 56.1 $\pm$ 3.8 & 56.6 $\pm$ 4.0 \\
MBPP      & 58.8 $\pm$ 1.9 & 59.5 $\pm$ 2.0 & 59.1 $\pm$ 1.8 \\
\bottomrule
\end{tabular}
\caption{Task performance on instruction-tuned models of baseline without length prediction compared with both DAEDAL and our survival-guided length selection method.}
\label{tab:inst-acc-daedal}
\end{table}

First, both methods preserve task performance relative to the corresponding baseline decoder. Second, our method achieves larger end-to-end speedups. This difference is expected from the structure of the methods: DAEDAL performs an iterative search over candidate lengths and depends on a thresholded stopping rule, whereas our estimator uses a single forward pass on a conservative canvas and converts the resulting per-position \texttt{[EOS]} probabilities into a closed-form expected length.

\subsection{Predictor Step Ablation}
\label{app:Predictor Step Ablation}

To further probe the stability of the length signal, we recompute the expected length from later denoising steps $s \in \{2,5,7,10\}$ and compare it to the estimate from the initial forward pass using LLaDa-Base model.

\begin{table}[!th]
\centering
\small
\setlength{\tabcolsep}{4pt}
\begin{tabular}{lcccc}
\toprule
\textbf{Task} & \textbf{Step 2} & \textbf{Step 5} & \textbf{Step 7} & \textbf{Step 10} \\
\midrule
BBH       & $0.5 \pm 2.3$ & $8.2 \pm 7.6$ & $7.2 \pm 8.6$ & $8.4 \pm 9.5$ \\
GSM8K     & $3.8 \pm 4.2$ & $4.1 \pm 5.6$ & $4.2 \pm 5.8$ & $4.2 \pm 6.3$ \\
MATH      & $-1.9 \pm 6.3$ & $7.0 \pm 6.7$ & $7.4 \pm 7.8$ & $8.0 \pm 9.3$ \\
HumanEval & $2.7 \pm 3.1$ & $4.5 \pm 8.3$ & $9.4 \pm 7.9$ & $6.9 \pm 7.9$ \\
MBPP      & $-0.8 \pm 1.2$ & $4.0 \pm 6.2$ & $8.2 \pm 8.1$ & $8.5 \pm 9.2$ \\
\bottomrule
\end{tabular}
\caption{Absolute change in predicted length $\delta \hat{L}^{(s)} = \hat{L}^{(s)} - \hat{L}^{(0)}$ when the estimator is recomputed from later denoising steps for LLaDA.}
\label{tab:length-drift}
\end{table}

Table~\ref{tab:length-drift} shows that the implied termination signal does not drift substantially during early decoding. These results support the use of a single-pass estimator and suggest that the initial \texttt{[EOS]} probabilities already contain most of the information needed for practical length selection.

\end{document}